\documentclass[runningheads]{llncs}
\usepackage{graphicx}

\usepackage{color}

\usepackage{hyperref}
\hypersetup{
	colorlinks = true,
    final=true, 
    plainpages=false,
    pdfstartview=FitV,
    pdftoolbar=true,
    pdfmenubar=true,
    pdfencoding=auto,
    psdextra,
    bookmarksopen=true,
    bookmarksnumbered=true,
    breaklinks=true,
    linktocpage=true,
    urlcolor=blue,
 	citecolor=blue,
 	linkcolor=blue,
 	pdfinfo={
 	    Title={CLIP: Cheap Lipschitz Training of Neuronal Networks},
	    Keywords={Deep neural network, 
	              Machine learning, Lipschitz constant,
	              Variational regularization,
	              Stability,
	              Adversarial attack}}
}
\usepackage{amsmath,amssymb}

\usepackage[ruled,vlined]{algorithm2e}
\usepackage{setspace}

\usepackage[capitalise,noabbrev]{cleveref}
\crefname{ass}{Assumption}{Assumptions}

\usepackage{booktabs}
\usepackage{multirow}

\usepackage{caption}
\usepackage{subcaption}

\usepackage{todonotes}
\usepackage{kantlipsum}
\usepackage{enumitem}

\newcommand{\revision}[1]{#1}

\newcommand{\R}{\mathbb{R}}
\newcommand{\N}{\mathbb{N}}
\newcommand{\lip}{\operatorname{Lip}}
\DeclareMathOperator*{\argmin}{arg\,min}

\newcommand{\abs}[1]{\left\vert#1\right\vert}
\newcommand{\norm}[1]{\left\Vert#1\right\Vert}
\newcommand{\st}{\,:\,}
\newcommand{\Lip}{\mathrm{Lip}}

\newcommand{\param}{\theta}

\newcommand{\net}{f}     
\newcommand{\inp}{x}
\newcommand{\inpp}{x'}
\newcommand{\oup}{y}
\newcommand{\oupp}{y'}
\newcommand{\Inp}{\mathcal{X}}
\newcommand{\Oup}{\mathcal{Y}}
\newcommand{\Param}{\Theta}

\newcommand{\trSet}{\mathcal{T}}
\newcommand{\trSetY}{\mathcal{T}_\Oup}
\newcommand{\trSetLip}{\mathcal{X}_\mathrm{Lip}}
\newcommand{\loss}{\ell}
\newcommand{\relu}{\operatorname{ReLU}}

\begin{document}
%
\title{CLIP: Cheap Lipschitz Training\\of Neural Networks\thanks{This work was supported by the European Union's Horizon 2020 research and innovation programme under the Marie Sk\l odowska-Curie grant agreement No. 777826 (NoMADS) and by the German Ministry of Science and Technology (BMBF) under grant agreement No. 05M2020 (DELETO).}}
\titlerunning{CLIP: Cheap Lipschitz Training of Neural Networks}
%
\author{Leon Bungert\inst{1} \and
 René Raab\inst{2} \and
 Tim Roith\inst{1} \and
 Leo Schwinn\inst{2} \and
 Daniel Tenbrinck\inst{1}}
%
\authorrunning{L. Bungert et al.}
%
\institute{Department Mathematics, Friedrich-Alexander University Erlangen-Nürnberg, Cauerstraße 11, 91058 Erlangen \\\email{\{leon.bungert,tim.roith,daniel.tenbrinck\}@fau.de} \and Machine Learning and Data Analytics Lab, Friedrich-Alexander University Erlangen-Nürnberg, Carl-Thiersch-Straße 2b, 91052 Erlangen\\
\email{\{rene.raab,leo.schwinn\}@fau.de}}
\maketitle              
\begin{abstract}
Despite the large success of deep neural networks (DNN) in recent years, most neural networks still lack mathematical guarantees in terms of stability. 
For instance, DNNs are vulnerable to small or even imperceptible input perturbations, so called adversarial examples, that can cause false predictions. 
This instability can have severe consequences in applications which influence the health and safety of humans, e.g., biomedical imaging or autonomous driving. 
While bounding the Lipschitz constant of a neural network improves stability, most methods rely on restricting the Lipschitz constants of each layer which gives a poor bound for the actual Lipschitz constant.

In this paper we investigate a variational regularization method named \emph{CLIP} for controlling the Lipschitz constant of a neural network, which can easily be integrated into the training procedure.
We mathematically analyze the proposed model, in particular discussing the impact of the chosen regularization parameter on the output of the network.
Finally, we numerically evaluate our method on both a nonlinear regression problem and the MNIST and Fashion-MNIST classification databases, and compare our results with a weight regularization approach.

\end{abstract}

\section{Introduction}
Deep neural networks (DNNs) have led to astonishing results in various fields, such as computer vision \cite{Krizhevsky2009} and language processing \cite{Oord2016}. 
Despite its large success, deep learning and the resulting neural network architectures also bear certain drawbacks.
On the one hand, their behaviour is mathematically not yet fully understood and there exist not many rigorous analytical results.
On the other hand, most trained networks are vulnerable to adversarial attacks \cite{Goodfellow2015}. 
In image processing, adversarial examples are small, typically imperceptible perturbations to the input that cause misclassifications. 
In domains like autonomous driving or healthcare this can potentially have fatal consequences. 
To mitigate these weaknesses, many methods were proposed to make neural networks more robust and reliable.
One straight-forward approach is to regularize the norms of the weight matrices \cite{gouk2020regularisation}.
Another idea is adversarial training \cite{Madry2018,schwinn2020rapid,Shafahi2019} which uses adversarial examples generated from the training data to increase robustness locally around the training samples.
In this context, also the Lipschitz constant of neural networks has attracted a lot of attention (e.g., \cite{combettes2020lipschitz,fazlyab2019efficient,szegedy2014intriguing,zou2019lipschitz}) since it constitutes a worst-case bound for their stability, and Lipschitz-regular networks are reported to have superior generalization properties \cite{oberman2018lipschitz}.
In \cite{terjek2019adversarial} Lipschitz regularization around training samples has been related to adversarial training.

Unfortunately, the Lipschitz constant of a neural network is NP-hard to compute \cite{scaman2018lipschitz}, hence several methods in the literature aim to achieve more stability of neural networks by bounding the Lipschitz constant of each individual layer \cite{Anil2019,gouk2020regularisation,Krishnan2020,Roth2020} or the activation functions \cite{aziznejad2020deep}.
However, this strategy is a very imprecise approximation to the real Lipschitz constant and hence the trained networks may suffer from inferior expressivity and can even be less robust~\cite{liang2020large}.
The following example demonstrates a representation of the absolute value function, which clearly has Lipschitz constant 1, using a neural network whose individual layers have larger Lipschitz constants (cf. \cite{huster2018limitations} for details).
\begin{align}
    \begin{tikzpicture}[baseline=(current  bounding  box.center)]
    \begin{scope}[every node/.style={circle,draw,minimum size=18pt,inner sep=0pt,outer sep=0pt}]
        \node (x) at (0,0) {$x$};
        \node (l1) at (3,0.7) {$x_+$};
        \node (l2) at (3,-0.7) {$x_-$};
        \node (ax) at (6,0) {$|x|$};
    \end{scope}
    \node[rectangle,draw,minimum height=2.5cm,dashed] at (3,0) {$\relu$};
    \begin{scope}[every path/.style={->,>=stealth}]
        \path (x) edge node[midway,above] {$1$} (l1);
        \path (x) edge node[midway,below] {$-1$} (l2);
        \path (l1) edge node[midway,above] {$1$} (ax);
        \path (l2) edge node[midway,below] {$1$} (ax);
    \end{scope}
    \end{tikzpicture}
\end{align}
In this paper we propose a new variational regularization method for training neural networks, while simultaneously controlling their Lipschitz constant.
Since the additional computations can easily be embedded in the standard training process and are fully parallelizable, we name our method \emph{Cheap Lipschitz Training of Neural Networks (CLIP)}.
Instead of bounding the Lipschitz constant of each layer individually as in prior approaches, we aim to minimize the global Lipschitz constant of the neural network via an additional regularization term.

The \textbf{main contributions} of this paper are as follows. First, we motivate and introduce the proposed variational regularization method for Lipschitz training of neural networks, which leads to a min-max problem.
For optimizing the proposed model we formulate a stochastic gradient descent-ascent method, the CLIP algorithm.
Subsequently, we perform mathematical analysis of the proposed variational regularization method by discussing existence of solutions.
Here we use techniques, developed for the analysis of variational regularization methods of inverse problems \cite{bungert2019solution,burger2013guide}.
We analyse the two limit cases, namely the regularization parameter tending to zero and infinity, where we prove convergence to a Lipschitz minimal fit of the data and a generalized barycenter, respectively.
Finally, we evaluate the proposed approach by performing numerical experiments on a regression example and the MNIST and Fashion-MNIST classification databases~\cite{LeCun98,Xiao2017}.
We show that the introduced variational regularization term leads to improved stability compared to networks without additional regularization or with layerwise Lipschitz regularization.
For this we investigate the impact of noise and adversarial attacks on the trained neural networks.
\section{Model and Algorithms}
Given a finite training set $\trSet\subset\Inp\times\Oup$, where $\Inp$ and $\Oup$ denote the input and output space, we propose to determine parameters $\param\in\Param$ of a neural network $\net_\param:\Inp\to\Oup$ by solving the empirical risk minimization problem with an additional Lipschitz regularization term
\begin{align}\label{eq:lipschitz_training}
\param_\lambda \in \argmin_{\param\in\Param} 
\frac{1}{\abs{\trSet}} 
\sum_{(\inp,\oup)\in\trSet}\loss(\net_\param(\inp),\oup) + \lambda\lip(\net_\param).
\end{align}
Here, $\loss:\Oup\times\Oup\to\R$ is a loss function, $\lambda>0$ is a regularization parameter and the Lipschitz constant of the network $\net_\param:\Inp\to\Oup$ is defined as
\begin{align}\label{eq:lipschitz_constant}
\lip(\net_\param):=
\sup_{\inp,\inpp\in\Inp}
\frac{\norm{\net_\param(\inp)-\net_\param(\inpp)}}{\norm{\inp-\inpp}}.
\end{align}
\revision{The norms involved in this definition can be chosen freely, however, for our CLIP algorithm we assume differentiability.%
}
The fundamental difference of the proposed model \eqref{eq:lipschitz_training} compared to existing approaches in the literature~\cite{Anil2019,gouk2020regularisation,Krishnan2020,Roth2020} is that we use the actual Lipschitz constant \eqref{eq:lipschitz_constant} as regularizer and do not rely on the layer-based upper bound
\begin{align}\label{ineq:lip_weights}
\lip(\net_\param) \ \leq \ \prod_{l=1}^L \lip(\Phi_l)
\end{align}
for neural networks of the form $f_\theta = \Phi_L \circ \dots \Phi_1$.
By plugging \eqref{eq:lipschitz_constant} into \eqref{eq:lipschitz_training} this becomes a min-max problem, which we solve numerically by using a
%
modification of the stochastic batch gradient descent algorithm with momentum $\mathrm{SGDM}_{\eta,\gamma}$ with learning rate $\eta>0$ and momentum $\gamma\geq 0$, see, e.g. \cite{Ruder16}. 

Since evaluating the Lipschitz constant of a neural network is NP-hard \cite{scaman2018lipschitz}, we approximate it \revision{on a finite subset} $\trSetLip\subset\Inp\times\Inp$ by setting
\begin{align}\label{eq:diff_quot}
    \Lip(\net_\param,\trSetLip) := \max_{(\inp,\inpp)\in\trSetLip} \frac{\norm{\net_\param(\inp)-\net_\param(\inpp)}}{\norm{\inp-\inpp}}.
\end{align}
To make sure that the tuples in $\trSetLip$ correspond to points with a high Lipschitz constant \revision{during the whole training process}, we update $\trSetLip$ using gradient ascent of the difference quotient in \eqref{eq:diff_quot} with \revision{adaptive step size}, see \cref{alg:adv_ud_basic}.
We call this \emph{adversarial update}, since the tuples in $\trSetLip$ approach points which have a small distance to each other while still getting classified differently by the network.

\begin{algorithm}[t!]
\setstretch{1.2}
\DontPrintSemicolon
\SetKwInOut{Input}{input}\SetKwInOut{Output}{output}
$L(\inp,\inpp) := {\norm{\net_\param(\inp)-\net_\param(\inpp)}}\,/\,{\norm{\inp-\inpp}},\quad\inp,\inpp\in\Inp$\;
\For{$(\inp,\inpp)\in\trSetLip$}{
%
$\inp\phantom{'} \gets \inp\phantom{'} + \tau\,L(\inp,\inp') \nabla_{\inp} L(\inp,\inpp)$\;
$\inpp \gets \inpp + \tau\,L(\inp,\inp') \nabla_{\inpp} L(\inp,\inpp)$\;
%
%
}
\caption{$\mathrm{AdversarialUpdate}_\tau$ (with step size $\tau>0$)}
\label{alg:adv_ud_basic}
\end{algorithm}
\begin{algorithm}[t!]
\setstretch{1.2}
\DontPrintSemicolon
\SetKwInOut{Input}{input}\SetKwInOut{Output}{output}
\For{epoch $e = 1$ \KwTo $E$}{
\For{minibatch $B\subset \trSet$}{
$\trSetLip \gets \mathrm{AdversarialUpdate}_\tau(\trSetLip)$\;
$\param \gets \mathrm{SGDM}_{\eta,\gamma}\left(\frac{1}{|B|}
\sum_{(\inp,\oup)\in B}\loss(\net_\param(\inp),\oup) + 
\lambda~\Lip(\net_\param,\trSetLip)\right)$\;
\If{$\mathcal{A}(\net_\param;\trSet) > \alpha$}{
$\lambda \gets \lambda + d\lambda$}
\Else{$\lambda \gets \lambda - d\lambda$}
}
}
\caption{CLIP (Cheap Lipschitz Training)}
\label{alg:lip_reg_basic}
\end{algorithm}

To illustrate the effect of adversarial updates, \cref{fig:adversarial_tuple} depicts a pair of images from the Fashion-MNIST database before and after applying Algorithm \ref{alg:adv_ud_basic}, using a neural network trained with CLIP (see~\cref{sec:classification} for details).
The second pair realizes a large Lipschitz constant and hence lies close to the decision boundary of the neural network.
\begin{figure}[bht]
\centering
\def\PicWidth{0.24\textwidth}
    \includegraphics[width=\PicWidth,trim=3.6cm 1.3cm 3.1cm 1.3cm,clip]{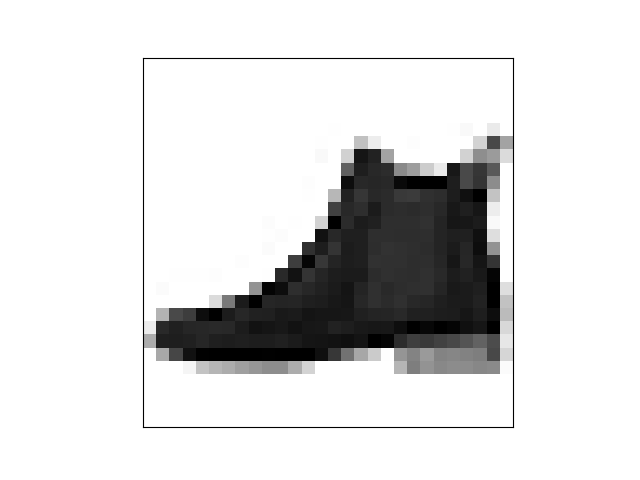}%
    \includegraphics[width=\PicWidth,trim=3.6cm 1.3cm 3.1cm 1.3cm,clip]{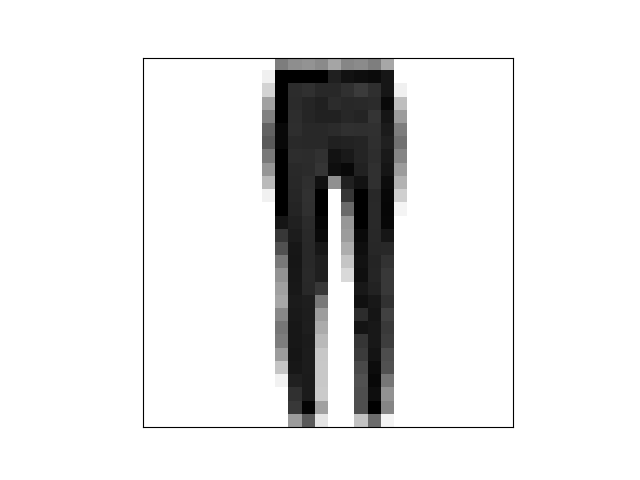}%
    \hfill%
    \includegraphics[width=\PicWidth,trim=3.6cm 1.3cm 3.1cm 1.3cm,clip]{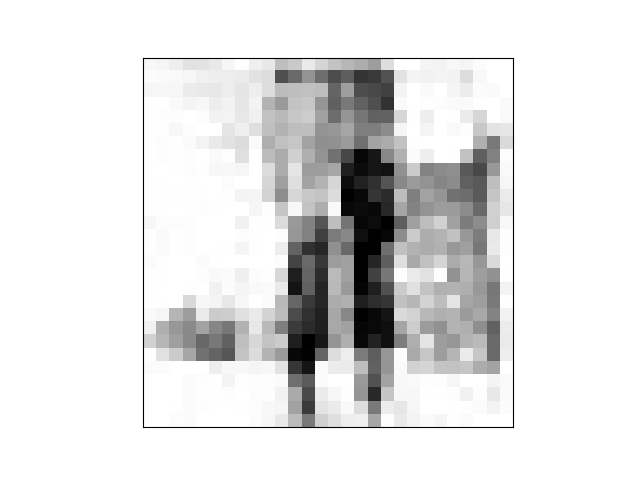}%
    \includegraphics[width=\PicWidth,trim=3.6cm 1.3cm 3.1cm 1.3cm,clip]{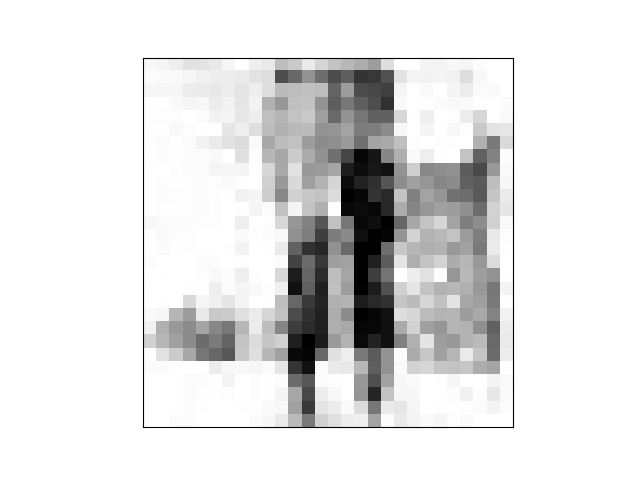}%
\caption{Effect of \cref{alg:adv_ud_basic}. The initial tuple gets classified correctly with high confidence ($>90\%$). After adversarial updates the pair gets misclassified as the respective other class with low confidence ($<30\%$).
\label{fig:adversarial_tuple}}
\end{figure}

The regularization parameter $\lambda$ in \eqref{eq:lipschitz_training} is chosen by a discrepancy principle from inverse problems \cite{anzengruber2009morozov,bungert2020variational}: We compare an accuracy measure $\mathcal{A}(\net_\param;\trSet)$ for the neural network $\net_\param$, evaluated on the training data, to a target accuracy $\alpha$. 
Combining the adversarial update from \cref{alg:adv_ud_basic} with stochastic gradient descent with momentum, we propose the CLIP algorithm, which is given in \cref{alg:lip_reg_basic}.
Similarly to \cite{Shafahi2019} one can reuse computations of the backward pass in the gradient step of the Lipschitz regularizer with respect to $\param$ in \cref{alg:lip_reg_basic} to compute the gradient with respect to the input $\inp$, needed in \cref{alg:adv_ud_basic}, which indeed makes CLIP a `cheap' extension to conventional training.
\revision{We would like to emphasize that the set $\trSetLip$, on which we approximate the Lipschitz constant of the neural network, is not fixed but is updated in an optimal way in every minibatch. Hence, it plays the role of an adaptive `training set' for the Lipschitz constant. Future analysis will investigate the approximation quality of this approach in the framework of stochastic gradient descent methods.}

\section{Analysis}
\label{sec:analysis}
In this section we prove analytical results on the model \eqref{eq:lipschitz_training}, which are inspired by analogous statements for variational regularization methods for inverse problems (see, e.g., \cite{bungert2019solution,burger2013guide}).
\revision{While the CLIP \cref{alg:adv_ud_basic} is stochastic in nature, our analysis focuses on the deterministic model \eqref{eq:lipschitz_training}, whose solutions are approximated by CLIP.}
For our results we have to pose some mild assumptions on the loss function $\loss$ and the neural network $\net_\param$ which are fulfilled for most loss functions and network architectures, \revision{such as mean squared error or cross entropy loss and architectures like feed-forward, convolutional, or residual networks with continuous activation functions.}
Furthermore, we assume that $\Inp,\Oup$, and $\Param$ are finite-dimensional spaces, which is the case in most applications and lets us state our theory more compactly than using Banach space topologies.
\begin{enumerate}[label= \textbf{Assumption \arabic*.}, ref={\arabic*}, align=left]
\setlength{\itemsep}{9pt}
\item\label[ass]{ass:loss} We assume that the loss function 
$\loss:\Oup\times\Oup\to\R\cup\{\infty\}$ satisfies:\\[-0.2cm]
\begin{enumerate}
\setlength{\itemsep}{6pt}
\item $\loss(\oup,\oupp)\geq 0$ for all $\oup,\oupp\in\Oup$,
\item $\oup\mapsto\loss(\oup,\oupp)$ is lower semi-continuous for all $\oupp\in\Oup$.
\end{enumerate}
\item\label[ass]{ass:net} We assume that the map $\param\mapsto\net_\param(\inp)$ 
is continuous for all $\inp\in\Inp$.
\item\label[ass]{ass:exist} We assume that there exists $\param\in\Param$ such that 
\begin{align*}
\frac{1}{\abs{\trSet}}
\sum_{(\inp,\oup)\in\trSet}\loss(\net_\param(\inp),\oup)+
\lambda \lip(\net_\param) < \infty.
\end{align*}
\end{enumerate}
We will also need the following lemma, which states that the Lipschitz constant of $\net_\param$ is lower semi-continuous with respect to the parameters $\param$.
\begin{lemma}\label{lem:lsc_lip}
Under \cref{ass:net} the functional $\param\mapsto\lip(\net_\param)$ is lower semi-continuous.
\end{lemma}
\begin{proof}
Let $(\param_i)_{i\in\N}\subset\Param$ converge to $\param\in\Param$.
Using the continuity of $\param\mapsto\net_\param(\inp)$ and the lower semi-continuity of the norm $\norm{\cdot}$, one can compute
\begin{align*}
    \lip(\net_\param) 
    \ &= \sup_{\inp,\inpp\in\Inp} \frac{\norm{\net_\param(\inp)-\net_\param(\inpp)}}{\norm{\inp-\inpp}} 
    \ \leq \ \sup_{\inp,\inpp\in\Inp} \liminf_{i\to\infty} \frac{\norm{\net_{\param_i}(\inp)-\net_{\param_i}(\inpp)}}{\norm{\inp-\inpp}} \\
    &\leq \ \liminf_{i\to\infty} \sup_{\inp,\inpp\in\Inp} \frac{\norm{\net_{\param_i}(\inp)-\net_{\param_i}(\inpp)}}{\norm{\inp-\inpp}} 
    \ = \ \liminf_{i\to\infty} \lip(\net_{\param_i}). 
\end{align*}
\end{proof}
\subsection{Existence of Solutions}
\label{sec:existence}
We start with an existence statement for a slight modification of~\eqref{eq:lipschitz_training}, which also regularizes the network parameters and guarantees coercivity of the objective.
\begin{proposition}
Under \cref{ass:loss,ass:net,ass:exist} the problem
\begin{align}\label{eq:lipschitz_training_reg}
\min_{\param\in\Param}
\frac{1}{\abs{\trSet}}
\sum_{(\inp,\oup)\in\trSet}
\loss(\net_\param(\inp),\oup) + 
\lambda\lip(\net_\param) + \mu\norm{\param}_\Param
\end{align}
has a solution for all values $\lambda,\mu>0$.
Here, $\norm{\cdot}_\Param$ denotes a norm on $\Param$.
\end{proposition}
\begin{proof}
We let $(\param_i)_{i\in\N}\subset\Param$ be a minimizing sequence, whose existence is assured by \cref{ass:exist}.
\revision{Since $\mu>0$ in \eqref{eq:lipschitz_training_reg}, the norms $\norm{\param_i}_\Param$ for $i\in\N$ are} uniformly bounded and hence, up to a subsequence which we do not relabel, $\param_i\to\param^*\in\Param$ as $i\to\infty$. 
The lower semi-continuity of $\loss$, $\lip$, and $\norm{\cdot}_\Param$ together with the continuity of $\param\mapsto\net_\param(\inp)$ then shows that $\param^*$ solves \eqref{eq:lipschitz_training_reg}.
\end{proof}
\begin{remark}
If $\Param$ is compact or even finite, existence for \eqref{eq:lipschitz_training} is assured since any minimizing sequence in $\Param$ is compact.
The reason that in the general case we can only show existence for the regularized problem \eqref{eq:lipschitz_training_reg} is that it assures boundedness of minimizing sequences.
Even for the unregularized empirical risk minimization problem existence is typically assumed in the \emph{realizability assumption} \cite{shalev2014understanding}.
\end{remark}
\subsection{Dependency on the Regularization Parameter}
\label{sec:dependency}
We start with the statement that the Lipschitz constant in \eqref{eq:lipschitz_training} decreases and the empirical risk increases as the regularization parameter $\lambda$ grows.
\begin{proposition}\label{prop:in_decrease_lambda}
Let $\param_{\lambda}$ solve \eqref{eq:lipschitz_training} for $\lambda\geq 0$.
Then it holds
\begin{alignat}{2}
\lambda&\longmapsto
\frac{1}{\abs{\trSet}}
\sum_{(\inp,\oup)\in\trSet}\loss(\net_{\param_\lambda}(\inp),\oup) \quad&&\text{is non-decreasing,}
\\
\lambda&\longmapsto\lip(\net_{\param_{\lambda}}) 
\quad&&\text{is non-increasing.}
\end{alignat}
\end{proposition}
\begin{proof}
The proof works precisely as in \cite{burger2013guide}.
\end{proof}
Now we will study the limit cases $\lambda\searrow 0$ and $\lambda\to\infty$ in \eqref{eq:lipschitz_training}.
As in Section~\ref{sec:existence}, because of lack of coercivity, we have to assume that the corresponding sequences of optimal network parameters converge.

Our first statement deals with the behavior as the regularization parameter $\lambda$ tends to zero. 
We show that in this case the learned neural networks converge to one which fits the training data with the smallest Lipschitz constant.
\begin{proposition}\label{prop:lambda_to_zero}
Let \cref{ass:loss,ass:net,ass:exist} and the realizability assumption~\cite{shalev2014understanding} 
\begin{align}\label{eq:training_fit}
\min_{\param\in\Param}
\frac{1}{\abs{\trSet}}
\sum_{(\inp,\oup)\in\trSet}\loss(\net_{\param}(\inp),\oup)=0
\end{align}
be satisfied.
Let $\param_\lambda$ be a solution of 
\eqref{eq:lipschitz_training} for $\lambda>0$. 
If $\param_\lambda \to \param^\dagger\in\Param$ as $\lambda\searrow 0$, 
then
\begin{align}\label{eq:param_dagger}
\param^\dagger\in\argmin\left\lbrace \lip(\net_\param) \st \param\in\Param,\, \frac{1}{\abs{\trSet}}
\sum_{(\inp,\oup)\in\trSet}\loss(\net_{\param}(\inp),\oup)
=0\right\rbrace
\end{align}
if this problem admits a solution with $\lip(\net_\param)<\infty$.
\end{proposition}
\begin{proof}
We fix $\param\in\Param$ which satisfies \eqref{eq:training_fit} and $\lip(\net_\param)<\infty$.
Using the optimality of $\param_\lambda$ we infer
\begin{align}\label{ineq:bounded_energy}
\liminf_{\lambda\searrow 0}
\frac{1}{\lambda\abs{\trSet}}
\sum_{(\inp,\oup)\in\trSet}\loss(\net_{\param_\lambda}(\inp),\oup) + 
\lip(\net_{\param_\lambda}) 
\ \leq \ \lip(\net_\param).
\end{align}
Using that $\loss$ is non-negative we conclude that $\lip(\net_{\param_\lambda})\leq \lip(\net_\param)$ and using Lemma~\ref{lem:lsc_lip} we get
\begin{align*}
\lip(\net_{\param^\dagger}) 
\ \leq \ 
\liminf_{\lambda\searrow 0} \lip(\net_{\param_\lambda}) \ \leq \ \lip(\net_\param).
\end{align*}
The lower semi-continuity of the loss $\loss$ and the continuity of the map $\param\mapsto\net_\param(\inp)$ for all $\inp\in\Inp$ imply
\begin{align*}
\loss(\net_{\param^\dagger}(\inp),\oup) \leq \liminf_{\lambda\searrow 0} \loss(\net_{\param_\lambda}(\inp),\oup),\quad\forall (x,y)\in\trSet.
\end{align*}
If we assume that $\loss(\net_{\param^\dagger}(\inp),\oup) > 0$ for some $(\inp,\oup)\in\trSet$, then we obtain 
\begin{align*}
    \liminf_{\lambda\searrow 0}
    \frac{1}{\lambda\abs{\trSet}}
    \sum_{(\inp,\oup)\in\trSet}\loss(\net_{\param_\lambda}(\inp),\oup) \geq \infty
\end{align*}
which is a contradiction to \eqref{ineq:bounded_energy} and implies that $\param^\dagger$ satisfies \eqref{eq:param_dagger}.
\end{proof}
\revision{Note that if the realizability assumption \eqref{eq:training_fit} is not satisfied, e.g., for noisy data or small network architectures, the previous statement has to be refined, which is subject to future work.}
The next proposition deals with the case in which the parameter $\lambda$ approaches infinity.
In this case the learned neural networks approach a constant map, coinciding with a generalized barycenter of the data. We denote by 
$\trSetY = \{y:(x,y)\in\trSet\}$ the set of the data in the output space.
\begin{proposition}\label{prop:lambda_to_inf}
Let \cref{ass:loss,ass:net,ass:exist} be satisfied and assume that 
\begin{align}\label{eq:manifold}
\mathcal{M}:=
\left\{\oup\in\Oup \st \exists\param\in\Param,\,
\net_\param(\inp)=y,\,\forall\inp\in\Inp\right\}\neq\emptyset.
\end{align}
Let $\param_\lambda$ denote a solution of \eqref{eq:lipschitz_training} for $\lambda>0$.
If $\param_\lambda \to \param_\infty \in\Param$ as $\lambda\to \infty$, then $\net_{\param_\infty}(\inp)=\hat{\oup}$ for all $\inp\in\Inp$ where 
\begin{align}\label{eq:barycenter}
\hat{\oup}\in\argmin\left\lbrace
\frac{1}{\abs{\trSet}}
\sum_{\oup\in\trSetY}\loss(\oup',\oup) 
\st \oup'\in\mathcal{M}\right\rbrace.
\end{align}
\end{proposition}
\begin{proof}
From the optimality of $\param_\lambda$ we deduce
\begin{align*}
\sum_{(\inp,\oup)\in\trSet}\loss(\net_{\param_\lambda}(\inp),\oup) + \lambda\lip(\net_{\param_\lambda}) 
\ \leq \ 
\sum_{\oup\in\trSetY}\loss(\oup',\oup),
\quad\forall\oup'\in\mathcal{M}.
\end{align*}
Hence, by letting $\lambda\to\infty$ we obtain $\lim_{\lambda\to\infty}\lip(\net_{\param_\lambda})=0$.
If we now assume that $\param_\lambda\to\param_\infty$ as $\lambda\to\infty$, we can use the lower semi-continuity of the Lipschitz constant to obtain
\begin{align*}
    \lip(\net_{\param_\infty}) \leq \liminf_{\lambda\to\infty}\lip(\net_{\param_\lambda}) = 0.
\end{align*}
Hence, $\net_{\param_\infty}\equiv \hat{\oup}$ for some element $\hat{\oup}\in\Oup$. 
Using the lower semi-continuity of the loss function, we can conclude the proof with
\begin{align*}
\sum_{\oup\in\trSetY}\loss(\hat{\oup},\oup)
\ &=  
\! \sum_{(\inp,\oup)\in\trSet}
\loss(\net_{\param_\infty}(\inp),\oup) 
\ \leq \
\liminf_{\lambda\to\infty}
\sum_{(\inp,\oup)\in\trSet}
\loss(\net_{\param_\lambda}(\inp),\oup) \\
&\leq \
\sum_{\oup\in\trSetY} \loss(\oupp,\oup),\quad\forall \oupp\in\mathcal{M}.
\end{align*}
\end{proof}
\begin{remark}
For a Euclidean loss function, i.e., $\loss(\oupp,\oup)=\norm{\oupp-\oup}^2$, the quantity \eqref{eq:barycenter} coincides with a projection of the barycenter of the training samples in $\trSetY$ onto the manifold of constant networks $\mathcal{M}$, given by \eqref{eq:manifold}.
This can be seen as follows: 
Let $b:=\frac{1}{\abs{\trSet}}\sum_{\oup\in\trSetY}\oup$ 
be the barycenter of the training samples.
Then
\begin{align*}
\frac{1}{\abs{\trSet}}\sum_{\oup\in\trSetY}\norm{\oupp-\oup}^2
&= 
\norm{\oupp}^2-2~\langle\oupp,b\rangle+
\frac{1}{\abs{\trSet}}\sum_{\oup\in\trSetY}\norm{\oup}^2\\
&=
\norm{\oupp-b}^2-\norm{b}^2+
\frac{1}{\abs{\trSet}}\sum_{\oup\in\trSetY}\norm{\oup}^2,
\qquad \forall \oupp\in\Oup.
\end{align*}
Using this equality for general $\oupp\in\mathcal{M}$ and for 
$\oupp=\hat{\oup}$ given by \eqref{eq:barycenter} we obtain
\begin{align*}
0 \ \leq \ \frac{1}{\abs{\trSet}}
\sum_{\oup\in\trSetY}\norm{\oupp-\oup}^2 
-\frac{1}{\abs{\trSet}}\sum_{\oup\in\trSetY}
\norm{\hat{\oup}-\oup}^2
\ = \ \norm{\oupp-b}^2 - \norm{\hat{\oup}-b}^2.
\end{align*}
This in particular implies that $\norm{\hat{\oup}-b}\leq\norm{{\oupp}-b}$ for all $\oupp\in\mathcal{M}$ as desired.
\end{remark}
\section{Experiments}

After having studied theoretical properties of the proposed variational Lipschitz regularization method \eqref{eq:lipschitz_training}, we will apply the CLIP algorithm to regression and classification tasks in this section. 
In a first experiment we train a neural network to approximate a nonlinear function given by noisy samples, where we illustrate that CLIP prevents strong oscillations, which appear in unregularized networks.
In a second experiment we use CLIP for training two different classification networks on the MNIST and Fashion-MNIST databases. 
Here, we quantitatively show that CLIP yields superior robustness to noise and adversarial attacks compared to unregularized and weight-regularized neural networks.
\revision{In all experiments we used Euclidean norms for the Lipschitz constant~\eqref{eq:lipschitz_constant}.}

Our implementation of CLIP is available on \texttt{github}.\footnote{\url{https://github.com/TimRoith/CLIP}}

\subsection{Nonlinear Regression}

In this experiment we train a neural network to approximate the nonlinear function $f(x)=\frac12\max(|x|-3,0)$ with three hidden layers consisting of 
500, 200 and 100 neurons using the \emph{sigmoid} activation function.
We construct 100 noisy training pairs by sampling values of $f$ on the set $[-4,-3]\cup[-0.3,0.3]\cup[3,4]$. 
The set of Lipschitz training samples is randomly chosen in every epoch and consists 
of tuples $(\inp,\inpp)$ where $\inp\in[-4,4]$ and $\inpp$ is a noise perturbation of $\inp$. 
\cref{fig:1D_regression} shows the learned networks after applying the CLIP Algorithm~\ref{alg:lip_reg_basic} with different values of the regularization parameter $\lambda$.
Note that for this experiment we do not use a target accuracy but choose a fixed value of $\lambda$.
For $\lambda=10$ the learned network is very smooth, which yields a poor approximation of the ground truth especially at the boundary of the domain.
For decreasing values of $\lambda$ the networks approach the ground truth solution steadily, showing no instabilities.
In contrast, the unregularized network with $\lambda=0$ shows strong oscillations in regions of missing data, which implies a poor generalization capability. 
\revision{Both for $\lambda=10^{-10}$ and $\lambda=0$ the trained networks do not fit the noisy data.
A reason for this is that some of the data points are negative and we use non-negative sigmoid activation in the last layer. Also, using stochastic gradient descent implicitly regularizes the problem and avoids convergence to spurious critical points of the loss.}

Note that we warmstart the computation for $\lambda=10$ with the pretrained unregularized network. 
The computations for $\lambda=1$ and $\lambda=10^{-10}$ are warmstarted with the respective larger value of $\lambda$.
The latter successive reduction of the regularization parameter is necessary in order to ``select'' the desired Lipschitz-minimal solution as $\lambda\to 0$, which resembles Bregman iterations~\cite{osher2005iterative}.
\begin{figure}[tbh]
    \def\PicWidth{0.24\textwidth}
    \centering
    \begin{subfigure}{\PicWidth}
    \centering
    \includegraphics[width=\textwidth,trim=0cm 0cm 0.1cm 0.1cm, clip]{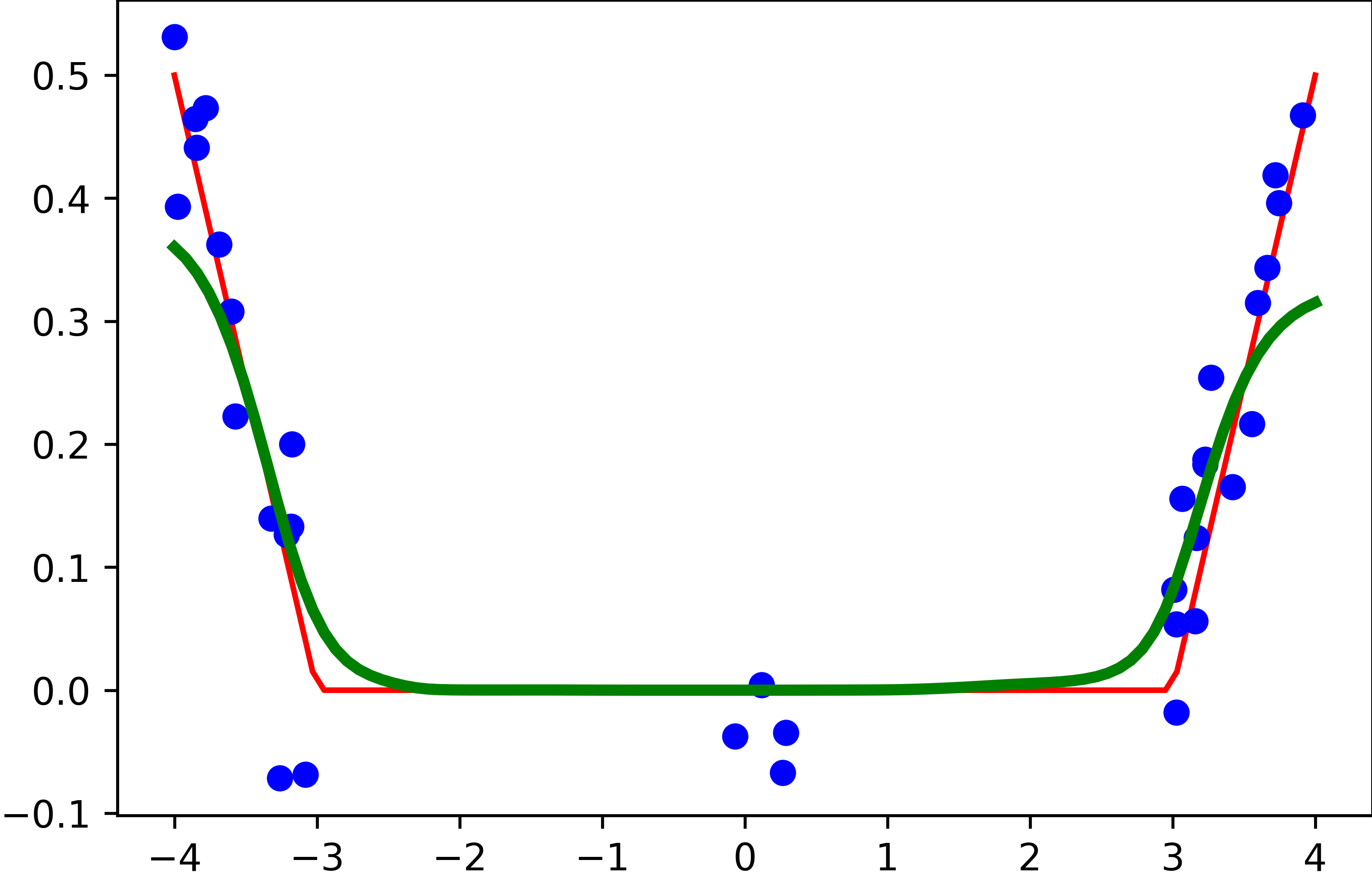}
    \caption{$\lambda=10$}
    \end{subfigure}%
    \hfill%
    \begin{subfigure}{\PicWidth}
    \centering
    \includegraphics[width=\textwidth,trim=0cm 0cm 0.1cm 0.1cm, clip]{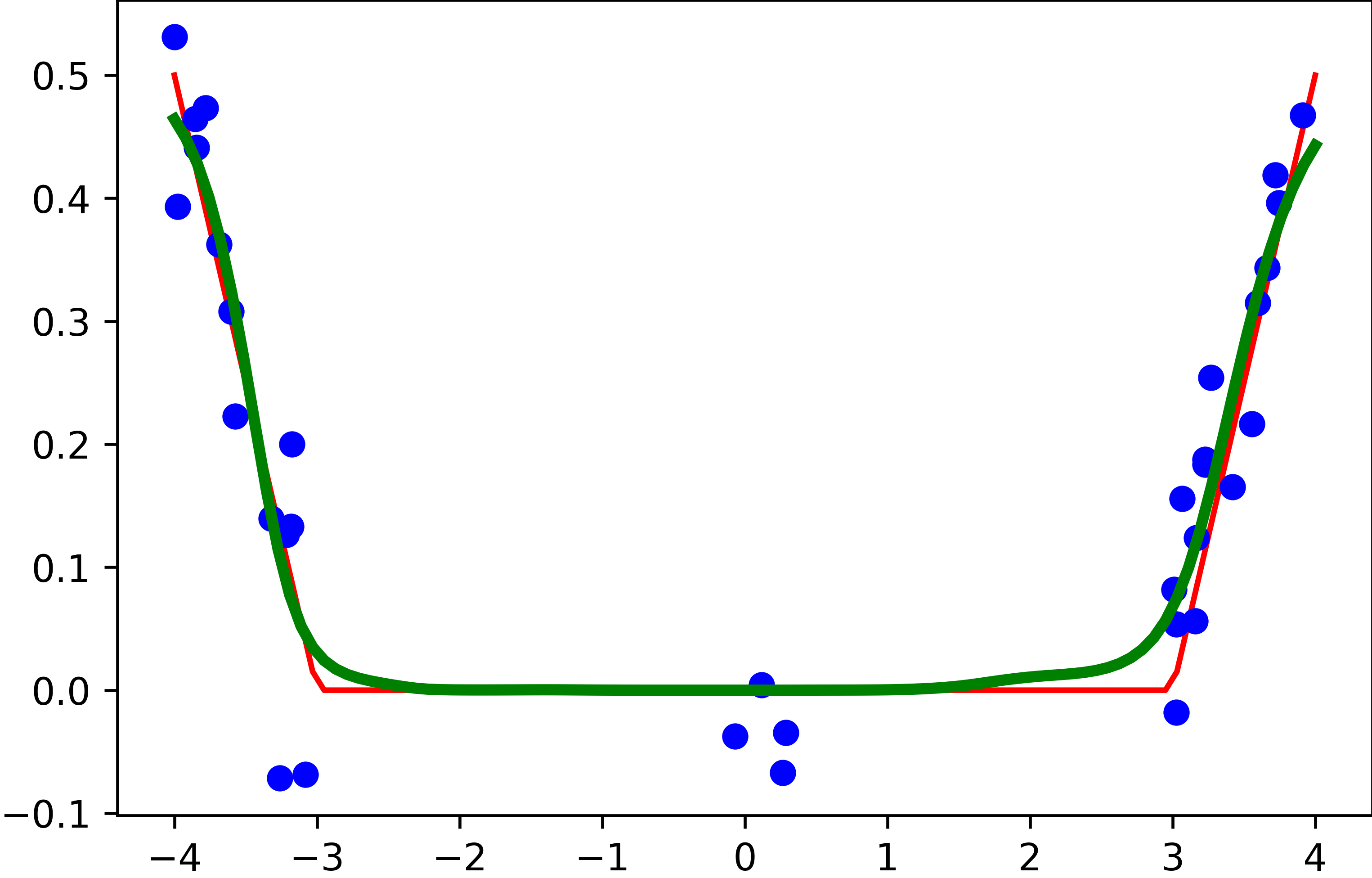}
    \caption{$\lambda=1$}
    \end{subfigure}%
    \hfill%
    \begin{subfigure}{\PicWidth}
    \centering
    \includegraphics[width=\textwidth,trim=0cm 0cm 0.1cm 0.1cm, clip]{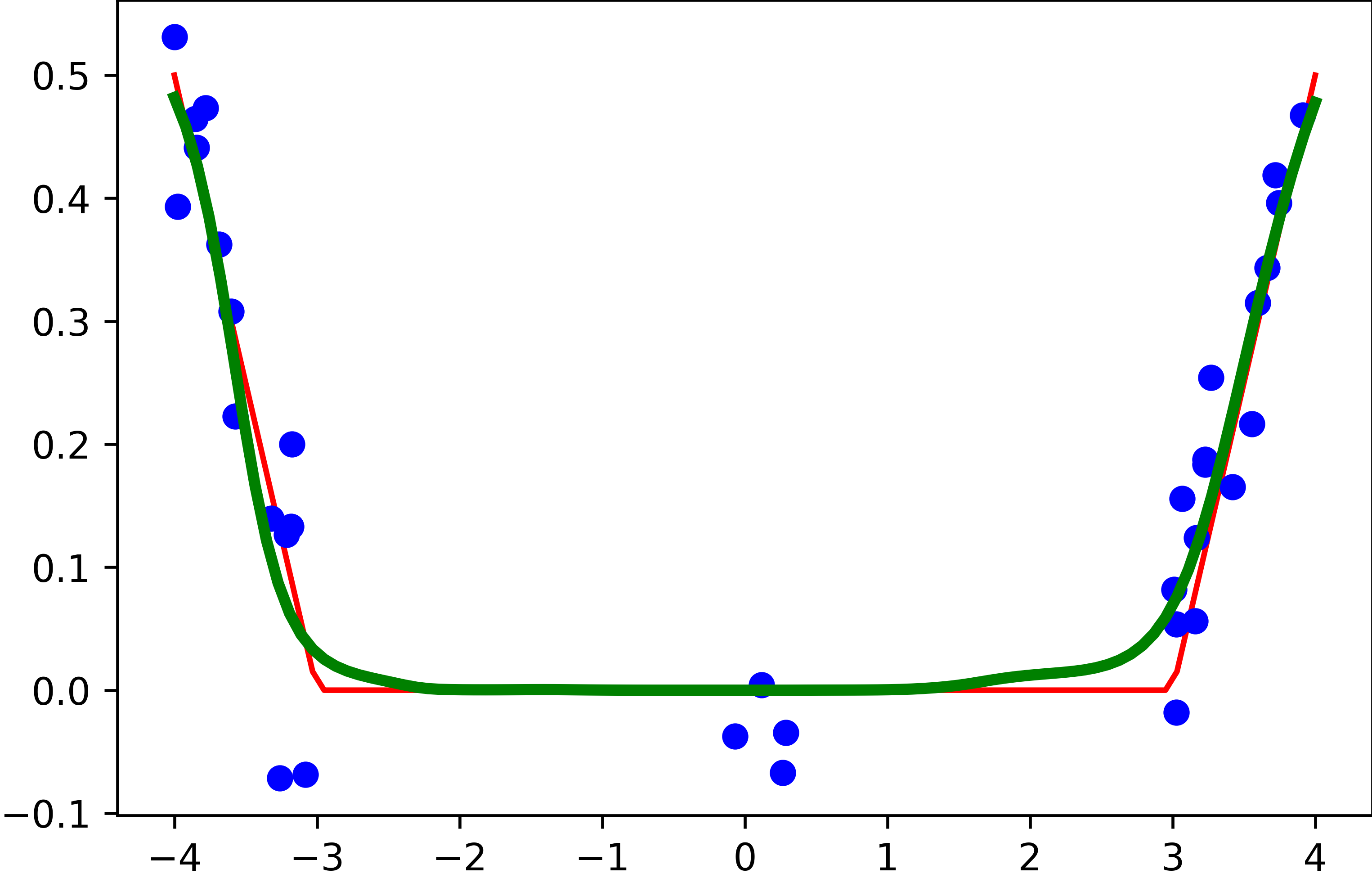}
    \caption{$\lambda=10^{-10}$}
    \end{subfigure}%
    \hfill%
    \begin{subfigure}{\PicWidth}
    \centering
    \includegraphics[width=\textwidth,trim=0cm 0cm 0.1cm 0.1cm, clip]{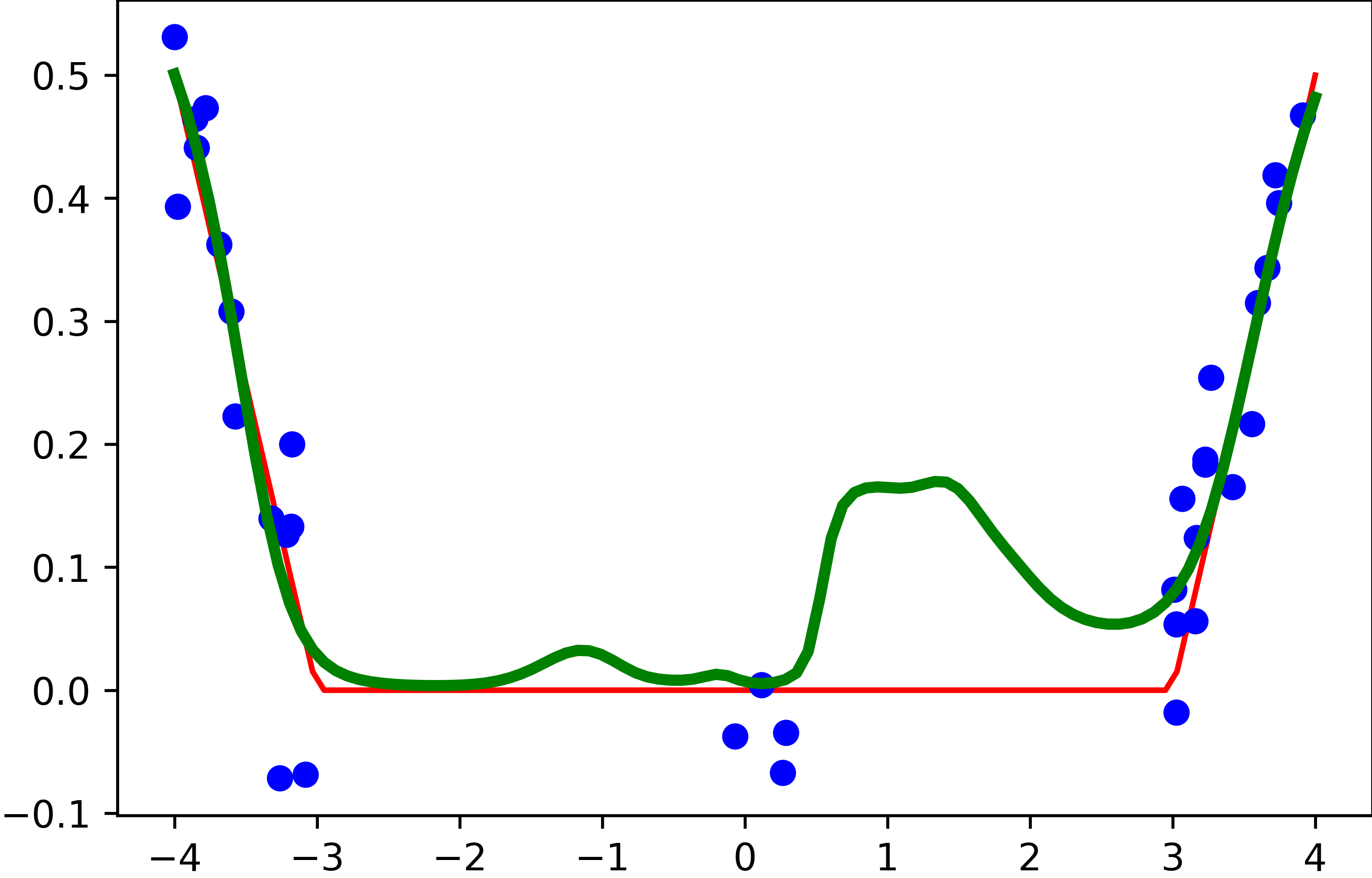}
    \caption{$\lambda=0$}
    \end{subfigure}%
    \caption{Lipschitz regularized networks for different values of $\lambda$. \textbf{Red:} ground truth. \textbf{Blue:} noisy ground truth samples. \textbf{Green:} trained neural network.}
    \label{fig:1D_regression}
\end{figure}
\begin{table}[h!]
\setlength{\tabcolsep}{1.1em}
    \centering
    \resizebox{0.92\textwidth}{!}{%
    \begin{tabular}{c|lcccc|r}
    \toprule
    &Training Type & Train & Test & Noise & PGD & $\mu$, $\lambda$  \\
    \midrule
    \multirow{7}{*}{\rotatebox{90}{\centering\textbf{MNIST}}}
    &Standard & \textbf{95.6} & 89.8 & 71.8 & 25.3 & 0 \\
    &Weight Reg.\textsubscript{95} & 93.1 & 89.0 & 71.4 & 25.4 & 0.0003 \\
    &Weight Reg.\textsubscript{90} & 89.8 & 89.7 & 71.9 & 24.8 & 0.0015 \\
    &Weight Reg.\textsubscript{85} & 84.8 & 87.2 & 72.3 & 24.9 & 0.0031 \\
    &{\bf CLIP\textsubscript{95}} & 95.3 & 90.7 & 70.3 & 24.4 & 0.01 \\ 
    &{\bf CLIP\textsubscript{90}} & 91.8 & \textbf{91.6} & \textbf{78.0} & 36.2 & 3.91 \\ 
    &{\bf CLIP\textsubscript{85}} & 87.6 & 87.2 & 74.5 & \textbf{37.7} & 9.82 \\ 
    \midrule
    \multirow{7}{*}{\rotatebox{90}{\textbf{Fashion-MNIST}}}
    &Standard & \textbf{98.5} & \textbf{92.0} & 25.7 & 1.9 & 0 \\
    &Weight Reg.\textsubscript{95} & 94.8 & 91.2 & 26.9 & 3.5 & 0.0011 \\
    &Weight Reg.\textsubscript{90} & 89.8 & 89.6 & 28.1 & 6.0 & 0.0023 \\
    &Weight Reg.\textsubscript{85} & 85.3 & 85.4 & \textbf{34.2} & 10.2 & 0.0091 \\
    &{\bf CLIP\textsubscript{95}} & 94.6 & \textbf{91.9} & 19.8 & 9.1 & 0.18 \\
    &{\bf CLIP\textsubscript{90}} & 90.6 & 88.9 & 23.3 & 13.0 & 0.74 \\
    &{\bf CLIP\textsubscript{85}} & 85.9 & 83.1 & 31.9 & \textbf{14.9} & 4.18 \\
    \bottomrule
    \end{tabular}
    }
    \caption{Accuracies [\%] for different training setups, tested on train, test, noisy, and adversarial \textbf{(Fashion-)MNIST} data. The target accuracies are given by subscript. The regularization parameters $\mu$ and $\lambda$ after training are also shown.}
    \label{tab:classification_results}%
    \vspace*{-.7cm}
\end{table}

\subsection{Classification on MNIST and Fashion-MNIST}
\label{sec:classification}
We evaluate the robustness of neural networks trained with the proposed variational CLIP algorithm on the popular MNIST \cite{LeCun98} and Fashion-MNIST \cite{Xiao2017} databases, which contain $28\times28$ grayscale images of handwritten digits and fashion articles, respectively.
They are split into $60,000$ samples for training and $10,000$ samples for evaluation. 
For MNIST we train a neural network with two hidden layers with sigmoid activation function containing 64 neurons each.
For Fashion-MNIST we use a simple convolutional neural network from \cite{Madry2018}.
%
We train non-regularized networks, CLIP regularized networks, and networks where we perform a $L_2$ regularization of the weights. 
The latter corresponds to~\eqref{eq:lipschitz_training_reg} with $\lambda=0$ and $\mu>0$, where we update $\mu$ with the same discrepancy principle as for CLIP. 
Note that the weight regularization implies a layerwise Lipschitz bound of the type \eqref{ineq:lip_weights}.
The quantity $\mathcal{A}(\net_\param;\trSet)$ in \cref{alg:lip_reg_basic} is chosen as the percentage of correctly classified training samples.
We set the target accuracies for $L_2$ regularized and CLIP regularized 
networks to $85\%$, $90\%$, and $95\%$.
The Lipschitz training set consists of 6,000 images removed from the training set and 
their noisy versions. 
We compare the robustness of all trained networks to adversarial perturbations $x+\delta$ of a sample $\inp\in\Inp$ in the $L_2$ norm, such that $\|\delta\| \leq \varepsilon$ with $\varepsilon=2$, calculated with the Projected Gradient Descent (PGD) attack \cite{Madry2018} as
\begin{align} \label{eq:pgd}
\inp_\mathrm{adv}^{t + 1} = \Pi\left(\inp_\mathrm{adv}^{t} + \alpha \operatorname{sign}(\nabla_{\inp}  \loss(\net_\param(\inp_\mathrm{adv}^{t}), \oup)\right),
\end{align}
where $\alpha>0$ is the step size and $\inp_\mathrm{adv}^{t}$ describes the adversarial example at iteration $t$. 
Here, $\Pi:\Inp\to\Inp$ is a projection operator onto the $L_2$-ball with radius $\varepsilon$, and $\operatorname{sign}$ is the componentwise signum operator. 
We set the step size $\alpha=0.25$ and the total number of iterations to $100$. 
Furthermore, we use Gaussian noise with zero mean and unit variance for an additional robustness evaluation. We track the performance against the PGD attack of each model during training and use the model checkpoint with the highest robustness for testing.

Table \ref{tab:classification_results} shows the networks' performance for the MNIST and Fashion-MNIST databases. The CLIP algorithm increases robustness with respect to Gaussian noise and PGD attacks depending on the target accuracy. 
One observes an inversely proportional relationship between the target accuracy and the magnitude of the regularization parameter $\lambda$ as expected from Proposition~\ref{prop:in_decrease_lambda}.
The results for MNIST indicate a
higher robustness of the CLIP regularized neural networks with respect to noise perturbations and adversarial attacks in comparison to unregularized or weight-regularized neural networks with similar accuracy on unperturbed test sets.
Similar observations can be made for the Fashion-MNIST database.
While our method is also more robust under adversarial attacks, it is slightly more prone to noise than the weight-regularized neural networks.
\section{Conclusion and Outlook}
We have proposed a variational regularization method to limit the Lipschitz constant of a neural network during training and derived \emph{CLIP}, a stochastic gradient descent-ascent training method. 
Furthermore, we have studied the dependency of the trained neural networks on the regularization parameter and investigated the limiting behavior as the parameter tends to zero and infinity.
Here, we have proved convergence to Lipschitz-minimal data fits or constant networks, respectively.
We have evaluated the CLIP algorithm on regression and classification tasks and showed that our method effectively increases the stability of the learned neural networks compared to weight regularization methods and unregularized neural networks.
\revision{In future work we will analyze convergence and theoretical guarantees of CLIP using techniques from stochastic analysis.}
%
%
%
%
%
%

\bibliographystyle{splncs04}
\bibliography{bibliography}
%

\end{document}